\documentclass[letterpaper]{article} 
\usepackage{arxiv}
\usepackage{times}  
\usepackage{helvet}  
\usepackage{courier}  
\usepackage[hyphens]{url}  
\usepackage{graphicx} 
\usepackage{natbib}  
\usepackage{caption} 
\usepackage{algorithm}
\usepackage{algorithmic}

\usepackage{amsmath,amsfonts,amssymb,mathtools,url,enumerate,verbatim}
\usepackage{amsthm}
\usepackage[german,british]{babel}
\usepackage{mathrsfs}
\usepackage{mathtools}
\usepackage{multirow}
\usepackage{subcaption}
\usepackage{comment}
\usepackage{authblk}
\usepackage{booktabs}
\usepackage{newfloat}
\usepackage{listings}
\DeclareCaptionStyle{ruled}{labelfont=normalfont,labelsep=colon,strut=off} 
\floatstyle{ruled}
\newfloat{listing}{tb}{lst}{}
\floatname{listing}{Listing}
    \newcommand\newdot{{\kern.8pt\cdot\kern.8pt}}
\newcommand\nbull{{\kern.8pt\raise1.5pt\hbox{\small\bf .}\kern.8pt}}
\newcommand\1{\hbox{\kern.375em\vrule height1.57ex depth-.1ex
		width.05em\kern-.375em \rm 1}}

\newcommand\R{\mathbb{R}}

\newcommand\rbb{\mathbb{R}}

\DeclareMathOperator*{\argmin}{arg\,min}
\DeclareMathOperator{\lip}{lip}

\DeclareMathOperator{\diam}{diam}

\newcommand\E{\mathbb{E}}

\DeclareMathOperator{\rad}{\mathfrak{R}}

\newtheorem{theorem}{Theorem}
\newtheorem{lemma}[theorem]{Lemma}
\newtheorem{proposition}[theorem]{Proposition}
\newtheorem{corollary}[theorem]{Corollary}
\theoremstyle{definition}
\newtheorem{definition}{Definition}

\theoremstyle{definition}

\usepackage{newfloat}
\usepackage{listings}
\DeclareCaptionStyle{ruled}{labelfont=normalfont,labelsep=colon,strut=off} 
\floatstyle{ruled}
\newfloat{listing}{tb}{lst}{}
\floatname{listing}{Listing}

\usepackage{bibentry}
\title{Explainable Neural Networks with Guarantees: A Sparse Estimation Approach}
\date{}

\author[]{\textbf{Antoine Ledent}}
\author[]{\textbf{Peng Liu} \thanks{\{aledent,liupeng\}@smu.edu.sg}}

\affil[]{Singapore  Management University}

\begin{document}

\maketitle
\def\thefootnote{*}\footnotetext{Both authors contributed equally to this work}
\begin{abstract}
Balancing predictive power and interpretability has long been a challenging research area, particularly in powerful yet complex models like neural networks, where nonlinearity obstructs direct interpretation. This paper introduces a novel approach to constructing an explainable neural network that harmonizes predictiveness and explainability. Our model, termed SparXnet, is designed as a linear combination of a sparse set of jointly learned features, each derived from a different trainable function applied to a single 1-dimensional input feature. Leveraging the ability to learn arbitrarily complex relationships, our neural network architecture enables automatic selection of a sparse set of important features, with the final prediction being a linear combination of rescaled versions of these features. We demonstrate the ability to select significant features while maintaining comparable predictive performance and direct interpretability through extensive experiments on synthetic and real-world datasets. We also provide theoretical analysis on the generalization bounds of our framework, which is favorably linear in the number of selected features and only logarithmic in the number of input features. We further lift any dependence of sample complexity on the number of parameters or the architectural details under very mild conditions. Our research paves the way for further research on sparse and explainable neural networks with guarantee.
\end{abstract}


\section{Introduction}

Neural networks have achieved state-of-the-art performance in multiple domains, from image and speech recognition~\citep{he2016deep,he2020resnet,lenet,inception} to natural language processing~\citep{devlin2018bert,church2017word2vec,tan2022exploring}. However, their complex architectures and the vast number of parameters mean that they can only be understood as `black box' models, where the decision-making process is not interpretable. Furthermore, neural networks typically involve millions of trainable parameters or even more, leading to poor understanding of their generalization behavior~\citep{Spectre}. This lack of transparency and explainability is bound to raise concerns, especially in high-stakes domains like healthcare, finance, and autonomous driving, where understanding the reasoning behind predictions is paramount.

Whilst many works attempt to analyze the predictions made by neural networks to make them more interpretable~\citep{Zhou2018InterpretingDV,DBLP:journals/corr/ZhouKLOT14,dhurandhar2018explanations,goyal2019counterfactual,varshneya2021learning}, the explanations produced are still far from the aim of providing an easily computable decision function which humans can understand and manipulate. For instance, in the computer vision literature, most works focus on visualizing concepts learned by individual neurons in intermediary layers. This doesn't fully explain how such concepts are learned from the pixel data or how these concepts are aggregated to produce a final prediction.  Similarly, in natural language processing, an interpretable method attempts to identify which words were given greater importance in generating the prediction; it is still difficult to fully explain how the model's prediction has utilized grammatical concepts or subtle clues such as irony. 

Although many of the issues above are arguably tied to the intrinsic complexity of typical machine learning problems such as computer vision and natural language processing, there is a lack of trustworthy and interpretable models in other domains, such as healthcare and finance, where the data are sometimes much lower-dimensional and where each feature corresponds to a concrete concept. These domains often require identifying a small set of relevant features, from a possibly high-dimensional space, that contribute most significantly to the outcome of interest. Reducing the dimensionality of input data also helps mitigate the curse of dimensionality and reduces the risk of overfitting. However, many feature selection methods often involve heuristics or statistical measures on feature ranking that are decoupled from the model training process.  Other methods, such as involve  $L^1$-norm regularization~\citep{roth2004generalized,zou2006adaptive,tibshirani2013lasso}, performs simultaneous feature selection and model estimation, but the resulting function typically depends on all selected features in an untractable way, making the final model less interpretable in complex models such as neural networks.

In this paper, we propose to tackle these issues by introducing an extremely parsimonious class of functions represented by neural networks. Our model SparXnet automatically selects a small number of important input features by applying a neuron-wise softmax transformation to all weights $W^k_u$ between the first and second layers for $u \in \{1,\cdots, d\}$ and $k \in \{1,\cdots, K\}$, where $d$ and $K$ denote the number of neurons in the first and second layers, respectively. Upon saturation, each neuron $k$ in the second layer will have one incoming weight close to one and the rest close to zero, thus selecting a single feature of the input data. In the next layer, SparXnet learns a separate one-dimensional function $f_k:\R \rightarrow \R$ for each selected feature. The final predictions (or logit scores in classification problems) are a linear combination of the outputs of each $f_k$. We represent the functions $\{f_k\}$ as deep neural networks, which are trained jointly in the entire model. 

The specific architecture of SparXnet has two immediate consequences: (1) predictions are highly interpretable: indeed, not only can we recover the selected features from the learned weights of the first layer and understand how much importance is given to each of them by looking at the weights of the last layer, but we can also plot the one-dimensional functions $\{f_k\}$ to investigate the individual effects each feature has on the final prediction, the same mechanism as the linear regression framework. For instance, suppose we are trying to predict the likelihood of heart disease in patients based on several factors such as BMI, number of cigarettes smoked per day, and blood pressure. If the model selects systolic blood pressure as one of the relevant features and the corresponding function $f_k$ exhibits a sudden sharp increase at 120 mm Hg, it indicates a threshold phenomenon around that value, thus providing a clear, interpretable threshold for assessing heart disease risk (2) Our model's function class capacity is drastically reduced compared to a traditional neural network. This is achieved in two ways: on the one hand, the feature selection forces the model to focus on a small number of features, which reduces the complexity of the model. Indeed, we provide generalization bounds for SparXnet, which indicate that the sample complexity is linear in the number of selected features, but only logarithmic in the number of features present in the data. In addition, with the very mild assumption that the functions $\{f_k\}$ are Lipschitz continuous, the sample complexity can completely avoid any dependence on the number of parameters or the architectural details of the models representing each of those functions, since the class of Lipschitz continuous functions itself has low sample complexity. Note that this would not be possible in higher dimensions, since the sample complexity of the space of Lipschitz functions is exponential in the input dimension. 

Our contributions can be summarized as follows: 
\begin{itemize}
	\item We propose an explainable neural network architecture that performs feature selection and model estimation simultaneously. Our model is parsimonious in that it only applies one-dimensional functions to each chosen feature before combining them with a linear layer. 
	\item We prove generalization bounds for SparXnet, which show a favorable sample complexity of $O\left( KL^2\log^3(d+L+1)\right)$, where $K$ is the number of selected features, $L$ is the Lipschitz constant of the learned one-dimensional functions, and $d$ is the original number of input features. In particular, the sample complexity of SparXnet doesn't involve the number of parameters of the networks used to represent the feature transformations $\{f_k\}$, or any other architectural details, and only depends logarithmically on the number of input features in the data. 
	\item To validate our method, we evaluate it on synthetic datasets with several noisy features and one informative feature. SparXnet exhibits superior performance compared to the standard neural network and successfully recovers the true feature and underlying one-dimensional function. In addition, the performance is relatively stable as we add more noisy features, whilst that of the standard neural network baseline deteriorates very fast.  
	\item Finally, we evaluate SparXnet on six real-life datasets, including adult income, breast cancer, credit risk, customer churn, heart disease, and recidivism. We achieve comparable or even superior results compared to feed-forward neural networks and other benchmark models, while preserving a much more interpretable and parsimonious model. We plot the feature transformation functions to further discuss the explainability of our model. 
\end{itemize}

The rest of this paper is organized as follows. In Section \ref{sec:relworks}, we discuss the related works and the improvements we seek to add to existing research. In Section~\ref{sec:methodology}, we introduce our notation and describe SparXnet in detail. In Section~\ref{sec:theory}, we prove the sample complexity guarantees for SparXnet. In Section~\ref{sec:experiments}, we discuss the results of our experiments on synthetic and real-world datasets. Finally, we conclude in Section~\ref{sec:conclusion}

\section{Related Works} \label{sec:relworks}

Making neural networks interpretable is a central concern in many machine learning applications~\citep{zhang2021survey}. For instance, in computer vision, many works attempt to interpret the specific type of features learned by each individual neuron by visualizing the associated representations~\citep{Simonyan2013DeepIC,DBLP:journals/corr/ZeilerF13}. This can then be compared with subjective human-understandable concepts~\citep{Zhou2018InterpretingDV,DBLP:journals/corr/ZhouKLOT14,DBLP:journals/corr/ZeilerF13}. Similarly, in natural language processing, some works attempt to interpret neural networks' representations in terms of the words they use~\citep{dalvi2019one}. In biological applications such as DNA sequence analysis, many authors attempt to interpret hidden representations as matching the search for certain explicit amino acid sequences~\citep{stormo1982use}. Other works focus on explaining neural network predictions through logic rules such as the presence or absence of certain specific features~\citep{dhurandhar2018explanations,goyal2019counterfactual,wachter2017counterfactual}. However, although several of those works focus on feature selection, none apply a trainable transformation function.  

The explanation of the performance of neural networks, despite their extremely high number of parameters, is a well-developed and active area of research. Earlier works focused on bounding the function class capacity of neural networks in terms of architectural parameters such as the number of parameters or the norms of the weights~\citep{longsed,Spectre,graf2022measuring,Ledent_21_Norm-based}. Since then, much of the literature has instead focused on the implicit regularization imposed by the gradient descent procedure and by the underlying structure in the data distribution~\citep{gradientover,NTK,Aroratwolayers,weima,inversepreac}. However, although many of these works rely on the Lipschitz constants of the network to bound complexity, none leverage the especially simple form of one-dimensional functions to sidestep the need for any other contributing terms (such as the norms of the weights or the complexity of the data distribution). It is worth noting that a particularly interesting new line of work~\citep{jacot2023implicit} has identified the phenomenon that neural networks with standard weight decay regularization may naturally restrict their `bottleneck rank': irrespective of the number of neurons present in each layer, a lower-dimensional representation of the input is naturally learned in intermediary layers. In particular, this indicates that in the case of a single function, standard neural networks may naturally strive to achieve a similar type of function as the ones learned by SparXnet. However, SparXnet involves several distinct one-dimensional feature-transformation functions rather than one single function whose input space has a small dimension that is still larger than one. In addition, ~\cite{jacot2023implicit} focuses on the training dynamics, while our work focuses on interpretability and generalization. 

The idea of using continuous one-dimensional feature transformations on several features goes back to early work in the statistics community on generalized additive models~\citep{hastie1986generalized, sardy2004amlet}. In particular, in projection pursuit regression~\citep{friedman1981projection}, a linear combination of features is fed through a nonlinear map, though unlike our work, no softmax is used to encourage the selection of a specific feature. Several distinguishing characteristics of our work compared to such early works are (1) the inclusion of generalization bounds from the point of view of modern statistical learning theory,  (2) the use of neural networks to model the nonlinear functions and (3) the introduction of feature selection with our softmax operator. 

A closely related line of research in modern literature is the Neural Additive Model (NAM) proposed by \cite{NAM}, which also addresses the challenge of achieving high predictive power while maintaining interpretability through generalized linear models. Our model differentiates itself from NAM in several key aspects. Firstly, we incorporate sparse estimation, which automatically selects the relevant set of features, thereby enhancing generalization performance. Furthermore, our work provides theoretical generalization bounds for the interpretable neural network, a contribution that is absent in previous studies.

\section{Methodology}
\label{sec:methodology}
Suppose our input data consists of i.i.d. samples $(x^1,y^1),$ $(x^2,y^2),\ldots,(x^N,y^N)$, where $y^i$ is the label and our inputs $x^i\in\R^d$ contains $d$ interpretable features. For instance, in our real data experiments on heart disease, examples of individual features include the resting heart rate (in beats per minute) and the average blood pressure. We assume that the individual features are suitably normalised so that $\|x^i\|_{\max}\leq \chi$  for some constant $\chi$. Our aim is to simultaneously select a small number $K\ll d$ of the input features and learn $K$ transformation functions $f^1,f^2,\ldots,f^K: \R\rightarrow \R$. Each function is represented as a deep neural network for ease of training and will be used to generate target prediction via a linear combination in the final layer. Thus SparXnet's prediction takes the following form:
\begin{align}
	\label{eq:thefirstequations}
F(x)=    \beta + \sum_{k=1}^K \theta_k f_k\left( \sum_{u=1}^d W^k_u x_u    \right).
\end{align}
where we assume an upper bound $\Gamma$ on $\sum_k |\theta_k|$. $F(x)$ is a scalar prediction in regression and logit score in classification, where the predicted probability for the positive class is $\frac{\exp(F(x))}{1+\exp(F(x))}$. $\beta$ is the bias term, $\{\theta_k\}_{k=1}^K$ denote parameters of the last linear layer, and $\{f_k\}_{k=1}^K$ are $K$ trainable functions represented by separate neural networks. Each $W^k_u$ is determined by the following formula: 
	\begin{align}
		W^k_{\nbull}&= \text{softmax}\left(   w^k_{\nbull}\right) \quad \quad \quad \text{i.e.} \nonumber\\ 
		W^k_u &=\frac{\exp(w^k_u/\tau)}{\sum_{v=1}^d\exp(w^k_v/\tau)} \quad \quad \quad (\forall u\leq d),
	\end{align}
where $\{w^k_{\nbull}\}$ are trainable parameters for the $k^{th}$ sub neural network and $\tau$ is a tunable temperature hyperparameter.

Thus, after softmax transformation, each row of the first-layer weight matrix $W$ represents a probability distribution reflecting the relative importance of each input feature for the corresponding neuron, thus prioritizing certain input features over others. As the model is trained, it learns the optimal weight distribution that minimizes the loss function, effectively performing feature selection and model estimation at the same time. In addition, the temperature parameter $\tau$ controls the `sharpness' of the softmax output. The higher the temperature, the more uniform the output distribution will be. Conversely, a lower temperature will make the distribution more sharply peaked. 

Figure~\ref{fig:model} illustrates the model architecture with two sub networks ($K=2$) with $x_1$ and $x_3$ being the selected input features\footnote{Our model allows the user to specify the number of features to remain, which offers a more precise control over model sparsity as opposed to indirectly configuring the penalty coefficient in Lasso regression.}. This involves two distinct processing pathways in a feed-forward neural network. The network enforces a softmax operation applied to the weights of the first hidden layer, where softmax serves as a soft form of `routing' mechanism, allowing the network to learn and distribute the representation of different data characteristics across the two pathways, thus achieving adaptive feature selection. For instance, the first input feature $x_1$ gets selected via a linear combination $\sum_u W^1_u x_u$  of the six input features with $W^1_1\gg W^1_u$ (for $u\neq 1$), where the dominating weight $W^1_1$ is denoted by the solid line (close to saturation) and the rest as dashed line. After softmax, the two nonlinear mapping functions $f_1(x_1)$ and $f_2(x_3)$ are automatically learned and linearly combined to generate the final prediction. Overall, training parameters include $\{w^k_u\}$ ($k\leq K$ and $u\leq d$) for the first two layers, $\beta$ and $\{\theta_k\}$ for the last two layers, and parameters of each individual fully connected network (FCN) $\{f_k\}$.

\begin{figure}
    \centering
    \includegraphics[width=1\linewidth]{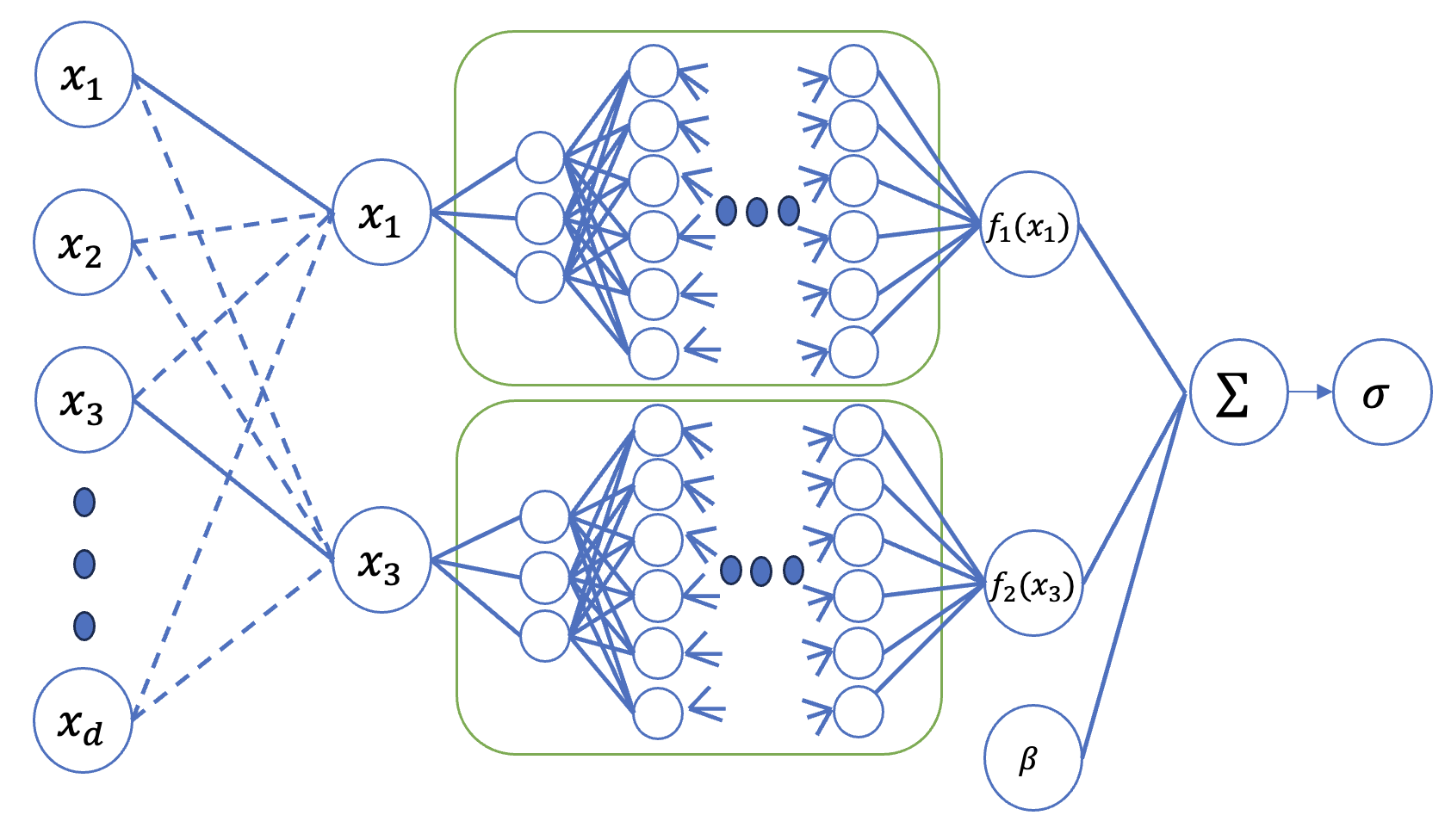}
    \caption{Schematic overview of the proposed model in the case of two selected features ($K=2$). The neural network has two distinct processing pathways, which are based on a softmax operation applied to the weights of the first hidden layer. The softmax operation serves as a soft form of ``routing" mechanism, allowing the network to learn and distribute the representation of different data characteristics across the two pathways, thus achieving adaptive feature selection upon saturation (as indicated by the two solid lines in the first layer). The two fully-connected mapping functions $f_1(x_1)$ and $f_2(x_3)$ are learned automatically and linearly combined to generate the final prediction.}
    \label{fig:model}
\end{figure}

\section{Theoretical Analysis}
\label{sec:theory}
In this section, we study the sample complexity of SparXnet. Our main result is Theorem~\ref{thm:maintheorem}, whose proof is left to the Appendix. 

\begin{theorem}
	\label{thm:maintheorem}
	
	Consider the function class $\mathcal{F}$ defined above. Suppose we are given $N$ i.i.d. samples $\{(x^i,y^i)\}_{i=1}^N$ with $y^i\in\R$ for all $i\leq N$ and a loss function $\ell:\R^2\rightarrow \R$ which is bounded by $B$ and has a Lipschitz constant at most $\mathcal{L}$. 
	Let 
	\begin{align}
		\hat{f}:= \argmin_{f\in\mathcal{F}} \frac{1}{N}\sum_{i=1}^N  \ell\left(f(x^i),y^i\right) 
	\end{align}
	and 
	\begin{align}
		f^*:= \argmin_{f\in\mathcal{F}} \E_{x,y} \ell\left(f(x),y\right).
	\end{align}
	
	We have, with probability $\geq 1-\delta$ over the draw of the training set: 
	\begin{align}
		&\frac{1}{N}\sum_{i=1}^N  \ell\left(\hat{f}(x^i),y^i\right)  -\E_{x,y} \ell\left(f^*(x),y\right)\nonumber \\&\leq \frac{24\mathcal{L}}{\sqrt{N}} \left[15\chi L\Gamma \sqrt{K}+3\right]\sqrt{\log_2\left(12dN^2\left[\chi L\Gamma  +1\right]\right) }\log(N)  +6B\sqrt{\frac{\log(2/\delta)}{2N}}
	\end{align}
\end{theorem}

We also obtain the following immediate corollary expressed in terms of excess risk $\epsilon$.

\begin{corollary}
	\label{cor:maincor}
	Assume the assumptions of Theorem~\ref{thm:maintheorem} hold with $\mathcal{L},\chi,\Gamma,B=O(1)$. The number of required samples to reach an excess risk of $\epsilon$ is $O\left( \frac{KL^2}{\epsilon^2} \log^3\left(\frac{KL^2 \log(d+L+1)}{\epsilon^2} \right)  \right)$.
\end{corollary}

Thus, for fixed $\chi,\mathcal{L},B$, the sample complexity is $\tilde{O}(L^2K)$: up to logarithmic terms, the number of samples required to train the model effectively is proportional to the product of the number of chosen factors $K$ and the square of the bound on the Lipschitzness constant. In particular, the dependency on the original number of features $d$ is only logarithmic. The proof strategy relies on the fact that the set of Lipschitz functions with a \textit{low-dimensional} input (1 or 2 d) has a very mild complexity~\citep{von2004distance,Tikhomirov1993}. This fact was previously exploited in the case of one dimensional inputs in several other contexts such as Matrix Completion~\citep{Ledent2024ICML} and density estimation~\citep{vandermeulen2021beyond}.

In summary, our bound has two particularities, which makes it non-vacuous compared to standard generalization bounds for neural networks: 
\begin{itemize}
	\item  There is no dependency on the number of parameters, since the complexity of the classes of 1 to 1 transformation functions is computed purely based on the Lipschitz constant.  This contrasts most of the literature on generalization bounds for neural networks, which almost always depend on the number of parameters of the model, whether the dependence is explicit~\citep{longsed} or implicit~\citep{Spectre,graf2022measuring,Ledent_21_Norm-based}.
	\item The only non-negligible dependency on architectural parameters is on $K$, not $d$ (although the bound on the norms of the original features is $\|x\|_{\max}$), indicating that the choice of a small number of important features from a multitude of features present in the original data has negligible cost in terms of sample complexity. This indicates that selecting a small number of important features from a high-dimensional input space does not significantly increase the sample complexity, making the model efficient and practical even in high-dimensional settings. This advantage is similar to the characteristic of generalization bounds for multi-class or multi-label prediction scenarios where the dependence is only logarithmic in the total number of possible labels, whilst maintaining nonlogarithmic dependence only on the number labels present~\citep{lei2018multi,structured,wu2021fine}.
\end{itemize}

Note that Corollary~\ref{cor:maincor} applies to both the regression and classification settings, with simple modifications of the loss function $\ell$. In the regression setting, we can use the truncated square loss:
\begin{align}
	\ell(\hat{y},y)=\min(|y-\hat{y}|^2,B).
\end{align}
 In the binary classification setting, we can use the cross entropy loss (with the softmax in the prediction step absorbed):
 \begin{align}
 	\ell(s,0)=\log\left( 1+\exp(s)  \right)  \quad 	\ell(s,1)=\log\left( 1+\exp(-s)  \right).\nonumber 
 \end{align}
 
For simplicity, our theoretical results are provided for the binary classification or the regression setting. However, the model can be straightforwardly extended to the multi-class case.

\section{Experimental Results}
\label{sec:experiments}

In this section, we present the experimental results using both synthetic and real data sets. We allocate 20\% of the data to the test set in all experiments, with the hyperparameters tuned by cross-validation. Our evaluation of the proposed algorithm serves a dual purpose: first, to gauge its predictive power, and second, to assess its capability to retrieve the sparse set of true features accurately. Through comprehensive analysis, we will illustrate that our method establishes a superior equilibrium between predictive accuracy and sparsity compared with prevalent benchmark techniques, thus positioning our approach as an innovative and effective solution for explainable neural networks with guarantee.

\subsection{Synthetic Data Experiments}

\subsubsection{A single-variable case}

In our first synthetic data experiment, we are interested in assessing whether the proposed method can recover the true functional form of the underlying relationship in the presence of different noise levels in both the input features and the observation model. Assuming a ground truth function $y_*=x^2+2sin(x)+3$, we generate 1000 observations uniformly sampled within the interval of [-1,1]. We inject additive Gaussian noise $\epsilon \in \mathcal{N}(0,0.05)$ into the observations and add a total of $J$ noisy features $x_i \in \mathcal{N}(0,1)$ for $i \in \{1,...,J\}$ into the feature space. This allows us to assess the model's capability to discern and recover the true signal from noise.

We use one pathway with six fully connected layers to learn the underlying data-generating process and identify the true feature. Each hidden layer consists of 128 nodes, followed by a dropout layer. We use Bayesian optimization to optimize three hyperparameters: dropout rate (between 0.1 and 0.5), learning rate (between 0.001 and 0.01), and temperature (between 0.1 and 100). We used a high temperature to promote early exploration, as a premature selection of an incorrect input feature during the early training phase may hurt the training performance. The temperature is then slowly reduced to 1\% of its initial value throughout a total training budget of 2000 iterations. 

Figure \ref{fig:synthetic_function_recovery} illustrates the learned function for this synthetic regression problem that includes one true feature and two noisy ones. We intentionally position the true feature in the middle of the design matrix to circumvent the possible default choice of selecting the first feature upon saturation. The figure suggests that SparXnet can recover the true shape of the underlying function, despite adversarial perturbation in both the observational and feature spaces. Specifically, SparXnet correctly selects the second input feature in the first layer, as evidenced by the learned weights of 4.84171210e-07, 9.99999523e-01 and 3.66482689e-09 in the first hidden layer.

\begin{figure}
  \includegraphics[width=\linewidth]{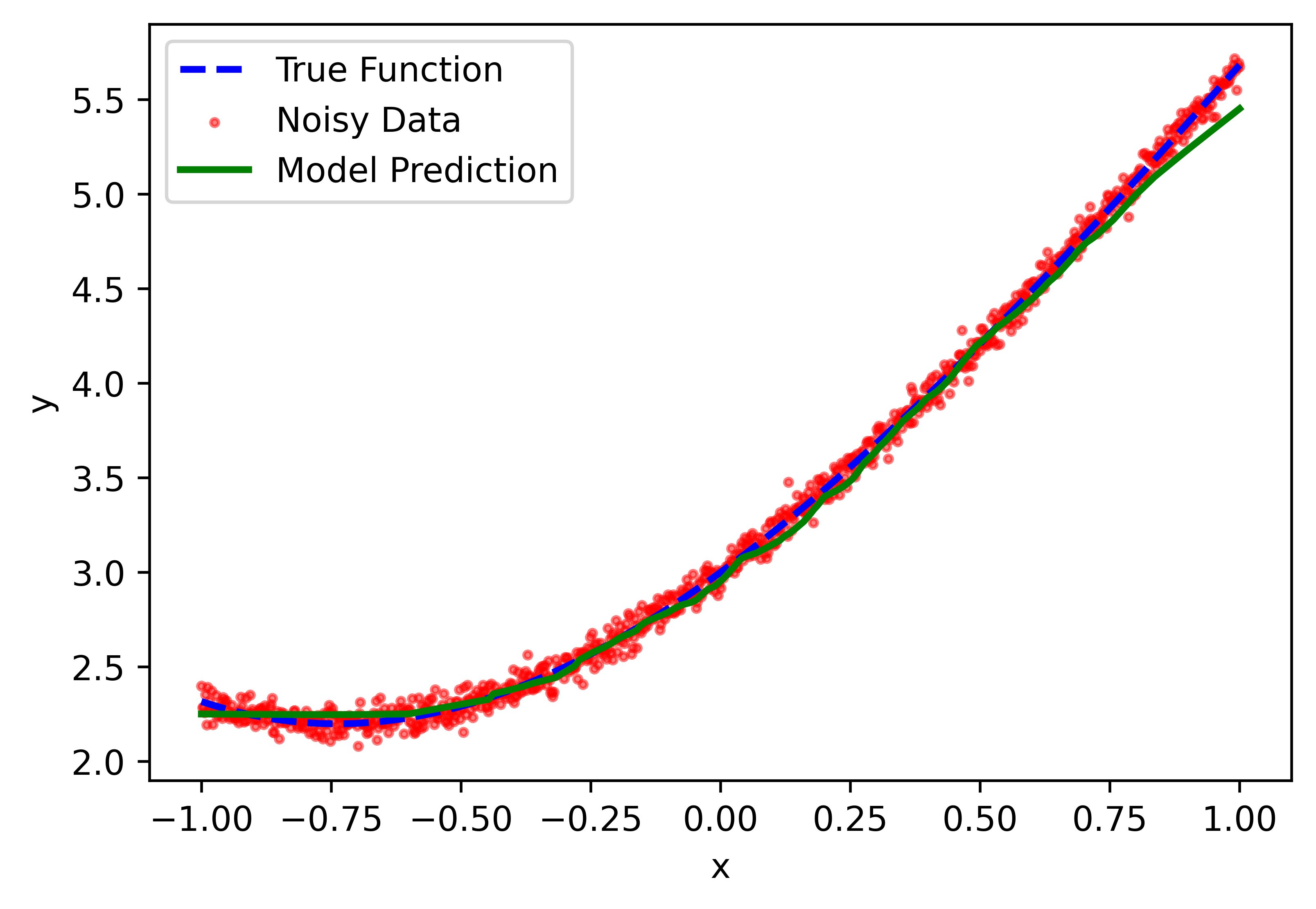}
  \caption{Visualizing the learned model predictions using 2000 noisy observations and three features, including one true feature and two random noise features. Our model can correctly identify the second feature as the true input feature (based on the learned weights of 4.84171210e-07, 9.99999523e-01 and 3.66482689e-09 in the first hidden layer) and recover the original form of the underlying data-generating function.}
  \label{fig:synthetic_function_recovery}
\end{figure}

To further assess our method under different signal-to-noise ratios, we progressively increase the number of noisy features from 2 to 5 while keeping the true feature in the central position of the design matrix. We use two benchmark methods for comparison: a fully connected neural network with the same number of layers and a Lasso regression model. As shown in Table \ref{synthetic_multiple_run} across five runs, SparXnet has a much lower test set MSE than others. This advantage is even more pronounced as we increase the number of noisy features in the data. We also observe that SparXnet can identify and recover the true feature in all experiments. This suggests that SparXnet can achieve good predictive performance while recovering the true feature simultaneously, when the ground truth is indeed sparse. The sparse solution also offers direct interpretability of the learned features, which will be further discussed in our next experiments.

\begin{table}[]
\caption{Comparing the mean and standard deviation of MSE for training and test sets across five runs. SparXnet has a much lower training and test set MSE than alternative models. This advantage is even more pronounced as we increase the number of noisy features in the data.}
\label{synthetic_multiple_run}
\resizebox{0.99\textwidth}{!}{ 
\begin{tabular}{|p{1.5cm}|p{2cm}|p{1.5cm}|p{1.5cm}|p{1.5cm}|p{1.5cm}|}
\hline
\multirow{2}{*}{Data type} & \multirow{2}{*}{Model} & \multicolumn{4}{c|}{Number of noisy features} \\ \cline{3-6} 
                           &                        & 2                        & 3                        & 4                        & 5                        \\ \hline
\multirow{3}{*}{Training}  & SparXnet               & \textbf{0.0047} (0.0017) & \textbf{0.0038} (0.0009) & \textbf{0.0047} (0.0013) & \textbf{0.0052} (0.0023) \\ \cline{2-6} 
                           & FCN                    & 0.0313 (0.0145)          & 0.2839 (0.5547)          & 0.0253 (0.0024)          & 0.4415 (0.5792)          \\ \cline{2-6} 
                           & Lasso regression model & 0.1248 (0.0012)          & 0.1245 (0.0019)          & 0.1251 (0.0023)          & 0.1244 (0.0015)          \\ \hline
\multirow{3}{*}{Test}      & SparXnet               & \textbf{0.0048} (0.0014) & \textbf{0.0039} (0.0009) & \textbf{0.0048} (0.0013) & \textbf{0.0054} (0.0022) \\ \cline{2-6} 
                           & FCN                    & 0.0327 (0.0159)          & 0.2681 (0.5204)          & 0.0269 (0.0034)          & 0.4207 (0.5393)          \\ \cline{2-6} 
                           & Lasso regression model & 0.1182 (0.0137)          & 0.1204 (0.0146)          & 0.1175 (0.0141)          & 0.1191 (0.0146)          \\ \hline
\end{tabular}
}
\end{table}

\subsubsection{A multi-variable case}

We now extend the results to a more challenging case with multiple variables to answer the following question: How will the recovery rate of the true features and the predictive power change as we vary the level of sparsity (controlled by the number of pathways) in SparXnet? This problem has a high practical relevance when one intends to select a subset of features for easy interpretation but is unsure of the exact number of true features present in the underlying model.

To this front, we assume a highly non-linear underlying function with five true features: $y_{\text{true}} = \sin(x_{0}) + 2x_{1}^{2} - 3x_{2}^{2} + 4e^{x_{3}} - 5e^{x_{4}}$. We also include five standard normal features, which are randomly arranged in a design matrix comprising 1000 observations. We track the true feature recovery rate and test-set MSE at different sparsity levels to evaluate the trade-off between these two objectives. Intuitively, one would expect a lower test-set MSE and a higher true feature recovery rate as more input features are present in the model (corresponding to a low sparsity level). However, having excessive noisy features will increase test-set MSE despite a high recovery rate of true features. 

Figure \ref{fig:synthetic_multi_var} illustrates the relationship between predictive accuracy and recovery rate at different sparsity levels in all models, including SparXnet, Lasso regression, Ridge regression, decision tree, and FCN. For test-set MSE, SparXnet performs comparably to Lasso when fewer than four features are retained and significantly outperforms Lasso as more features are added to the final model by adjusting the sparsity level. When using all available input features, SparXnet only performs inferior to FCN, highlighting SparXnet's efficiency in extracting the most informative features for prediction.

Moreover, SparXnet demonstrates excellent performance in recovering the true features, often matching or exceeding the recovery rates of Lasso. An exception occurs when eight features are selected, where SparXnet slightly lags behind Lasso. This can be attributed to a potential overlapping coverage of features chosen by SparXnet, a complexity absent in Lasso due to its non-overlapping selection mechanism. It is also important to note that all other benchmark methods rely on the full set of features and, therefore, offer no sparsity.

\begin{figure}[htb!]
  \includegraphics[width=1\linewidth]{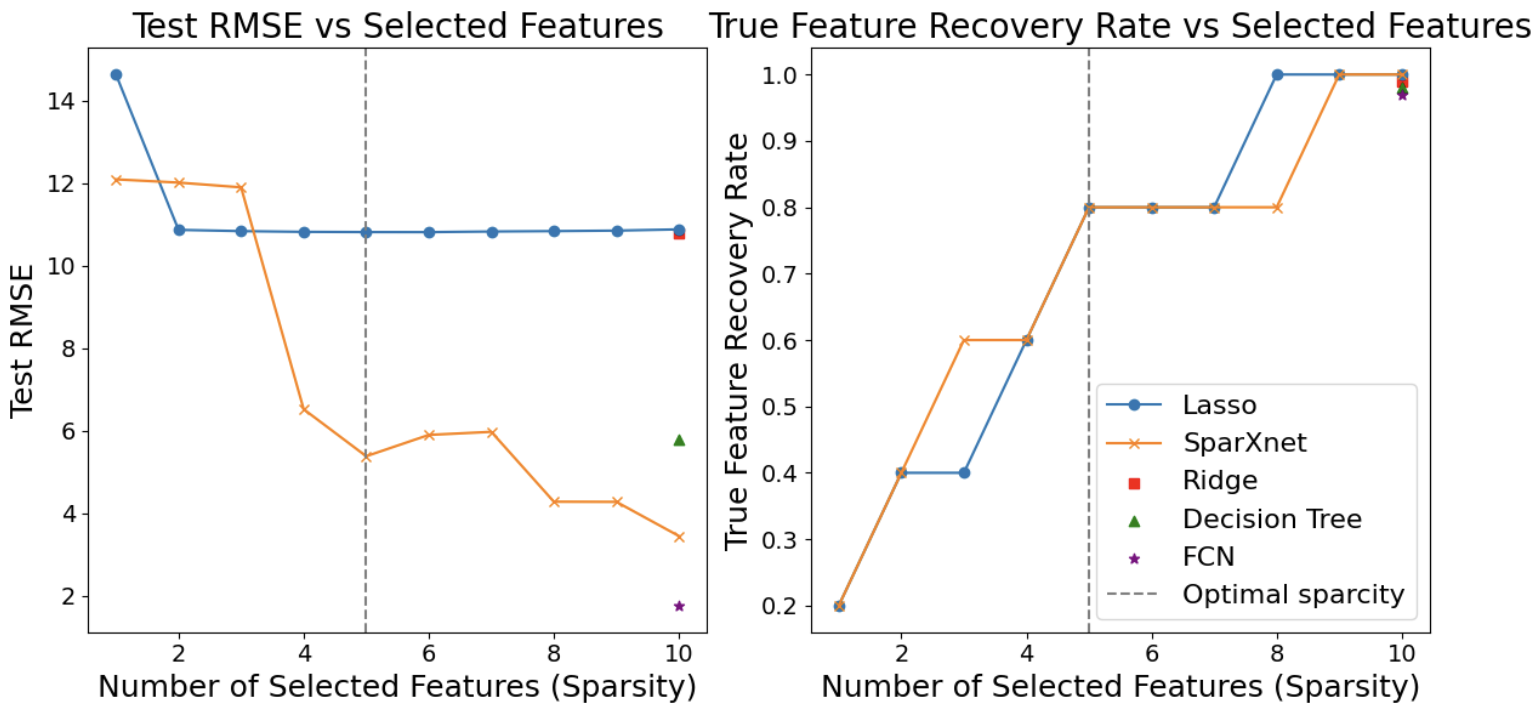}
  \caption{The dynamics between predictive accuracy and recovery rate of true features as the sparsity level varies across different models, including SparXnet, Lasso, Ridge regression, decision tree, and FCN. At high sparsity levels (fewer than four features), SparXnet has a higher test RMSE than Lasso. However, as the number of features increases, SparXnet exhibits enhanced performance, eventually surpassing Ridge regression and decision tree models, and approaching the performance of FCN. SparXnet also demonstrates superior or comparable recovery rates of true features relative to Lasso, except when eight features are selected, where potential overlapping coverage of features chosen by SparXnet slightly reduces its recovery rate.}
  \label{fig:synthetic_multi_var}
\end{figure}

\subsection{Real Data Experiments}

The sparse estimation approach shines in applications where one prefers a small subset of interpretable features. For example, a credit officer needs to rely on a small set of important features to make a credit lending decision, which must also be fully explainable when making pass/fail decisions. Similarly, a doctor often attributes disease to a small number of causes when explaining it to patients. Therefore, our experiments with real data are designed to cover these real-world scenarios that emphasize both predictiveness and explainability in a classification setting. We also assess a wider pool of benchmark models, including logistic regression, FCN, NAM, decision tree, and XGBoost. 

We implemented a consistent network architecture for all six datasets, including adult income, breast cancer, credit risk, customer churn, heart disease, and recidivism. We varied the number of nodes $K$ in the first hidden layer, corresponding to the number of features the model would select. In each set of experiments, we optimized the hyperparameters for SparXnet, including the learning rate, batch size, configurations of hidden layers, dropout rate, and the seed for the train-test split using Bayesian optimization based on the validation set. We then evaluated the optimal model configurations on the test set. Table \ref{tab:model_performance} displays the mean and standard deviation of the test set AUC from the top 5 (lowest validation loss) out of 30 repeated experiments. In particular, SparXnet consistently ranks among the top three models in all datasets, demonstrating robust predictive performance in various contexts. 

\begin{table*}[ht]
\centering
\caption{Average out-of-sample AUC, with standard deviations in parentheses, for various predictive models across different datasets. Notably, SparXnet consistently achieves top-three performance across all evaluated datasets.}
\label{tab:model_performance}
\begin{tabular}{|l|l|l|l|l|l|l|}
\hline
\textbf{Dataset} & \textbf{FCN}& \textbf{NAM}& \textbf{Log. Regr.} & \textbf{Decision Tree} & \textbf{XGBoost} & \textbf{SparXnet} \\
\hline
Adult & 0.887 (0.015) & 0.900 (0.003) & 0.854 (0.000) & 0.897 (0.002) & 0.924 (0.002) & 0.899 (0.008) \\
Breast & 0.988 (0.005) & 0.645 (0.088) & 0.998 (0.000) & 0.953 (0.004) & 0.995 (0.001) & 0.989 (0.010) \\
Credit Risk & 0.893 (0.011) & 0.849 (0.004) & 0.851 (0.000) & 0.899 (0.012) & 0.949 (0.002) & 0.910 (0.003) \\
Cust. Churn & 0.754 (0.016) & 0.784 (0.056) & 0.837 (0.000) & 0.757 (0.012) & 0.838 (0.003) & 0.832 (0.009) \\
Heart Disease & 0.874 (0.027) & 0.546 (0.076) & 0.945 (0.001) & 0.789 (0.044) & 0.950 (0.004) & 0.853 (0.030) \\
Recidivism & 0.668 (0.012) & 0.669 (0.024) & 0.716 (0.000) & 0.625 (0.004) & 0.715 (0.005) & 0.703 (0.010) \\
\hline
\end{tabular}
\end{table*}

To assess the impact of the number of pathways in the first hidden layer on model performance, we analyzed the relationship between the number of pathways and the average test set MSE. Specifically, we selected the top five models with the lowest validation errors for each number of pathways, while also ensuring that the number of pathways did not exceed the total number of features available for each dataset. See the Appendix for further analysis on the relationship between remaining number of features and out-of-sample MSE.

To further highlight the uniqueness of our approach in simultaneous model estimation and feature selection, we provide an example of the inference procedure for two sample applicants from the credit risk dataset. Figure \ref{fig:model_inference_v2} demonstrates the inference examples corresponding to high-risk and low-risk applicants, respectively, showcasing the distribution of individual features and the predicted probability output. We selected six pathways with the highest marginal increase in generalization performance, as shown in the appendix, and plot the density plot of each pathway before and after their respective transformations.

In this experiment, SparXnet identified several important features, most notably the loan’s interest rate and the applicant’s annual income. The results can be visualised in Figure \ref{fig:model_inference_v2}. As expected, the model predicts that higher interest rates are associated with a higher probability of default, while higher applicant income correlates with a lower probability of default. Interestingly, in both cases, the model detects thresholds where the risk sharply increases. For instance, interest rates above 20 percent are linked to a very high probability of default. Regarding annual income, the initial portion of the graph (corresponding to lower-income applicants) shows a steady decrease in risk as income increases. For middle earners, the risk remains relatively constant, then drops to a negligible level in the high-income bracket.

Note that the model was trained with an additional temperature parameter to promote the saturation of the weights learned in the first layer, thus accelerating feature selection. Indeed, unimportant features show zero or near-zero weights, demonstrating the efficacy of this feature selection mechanism. The model, trained with the same hyperparameters as the best model with six pathways and the lowest validation loss, achieves an AUC of 0.82, which is comparable to alternative models. See more details in the Appendix on weight saturation in the softmax layer. 

Our framework also facilitates transparent tracking of the decision-making process. This prediction results from a linear combination of transformed pathways and a softmax operation in the last layer. The score is then contrasted with those of the broader population, facilitating a comprehensive credit risk assessment \citep{liu2023integrated}. For instance, in Figure \ref{fig:model_inference_v2}, the high-risk applicant, indicated via the red line, is applying for a very high interst loan, which results in a high default probability, despite the applicant's income being in the `safe' range. In contrast, the low-risk applicant, indicated in blue, remains within safe ranges for both features, resulting in a low default probability.

\begin{center}
\begin{figure}[h!]
	\includegraphics[width=.95\linewidth]{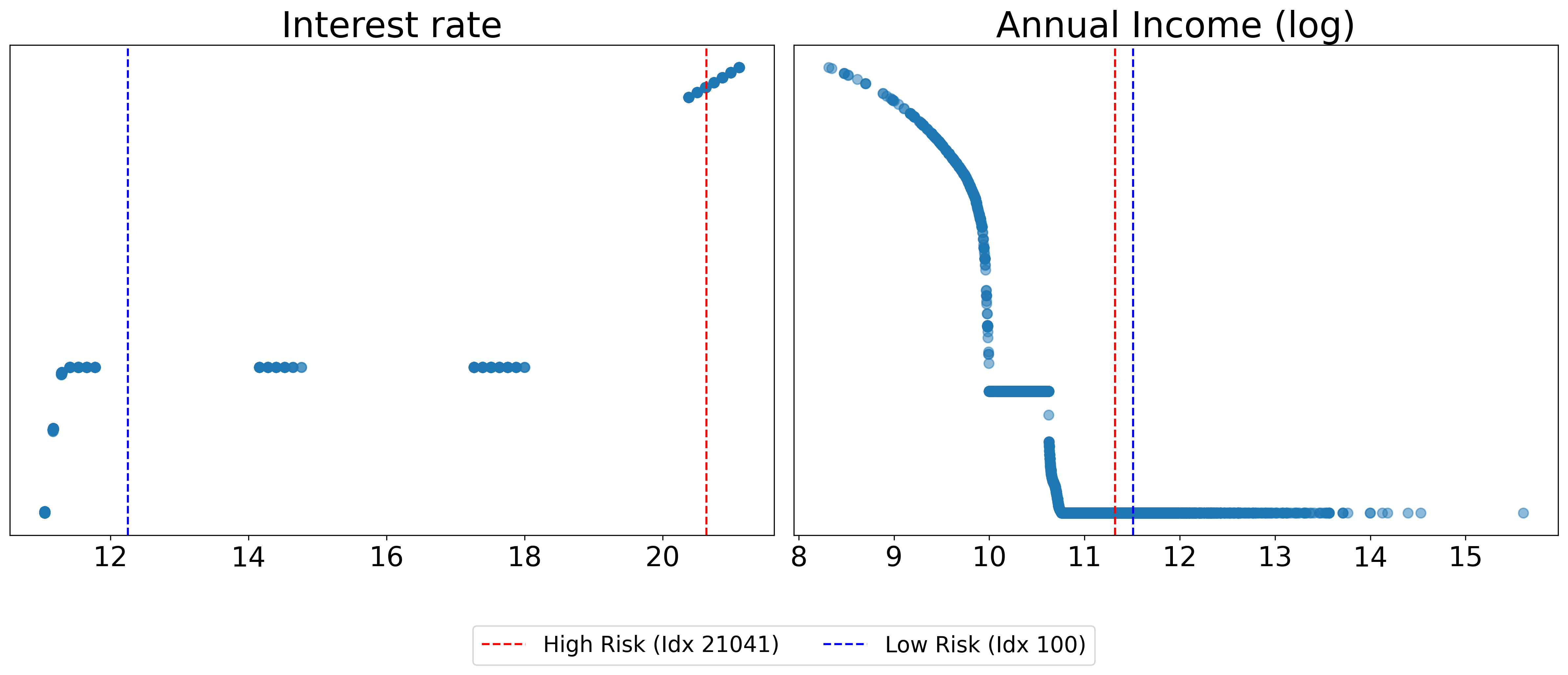}
		\includegraphics[width=.95\linewidth]{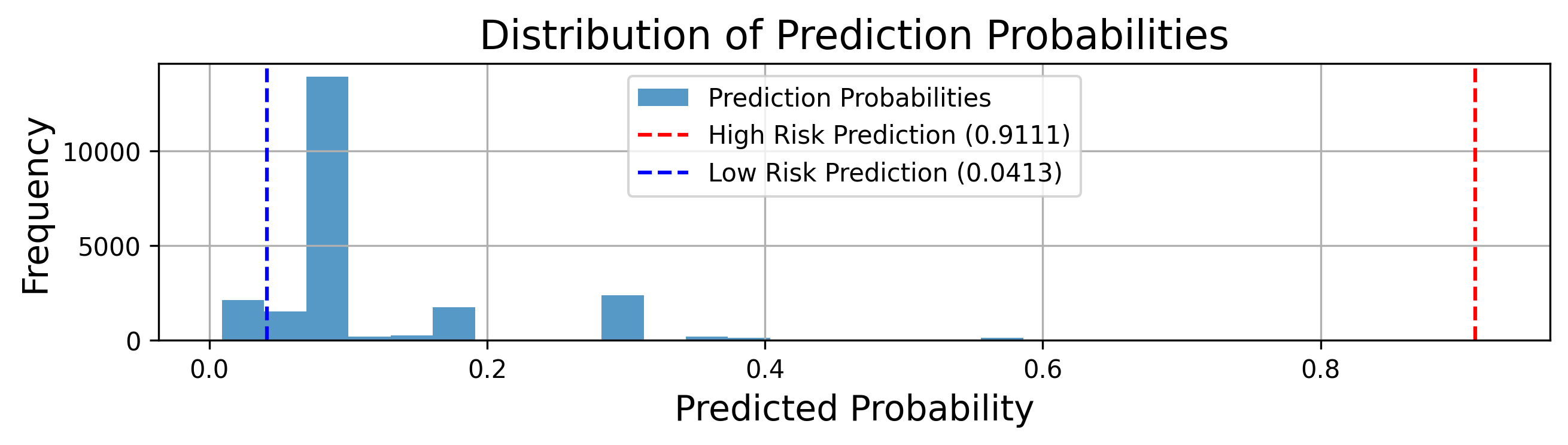}
          \caption{Illustrating the model inference for two credit applicants in the credit risk dataset.}
    \label{fig:model_inference_v2}
    \end{figure}
\end{center}

Note that this ability to assess individual feature-level critical regions during inference showcases the strength of SparXnet in explainable sparse estimation. Importantly, this inferential framework enables the interpretation of factors driving high-risk applicants. One possible downside to the approach is that the softmax may not always saturate, leading to some of the feature selection nodes incorporating multiple features. Whilst this can improve representation power and accuracy~\citep{liu2024kan}, this is done at the cost of some of the interpretability gains that define SparXnet.  Mitigating this phenomenon through a more effective tuning of the temperature parameter is an interesting avenue of research for future work. In addition, the smoothness of some of the predictive feature transformations could be further improved with gradient regularization. Lastly, we acknowledge that the arbitrary linear combination involved at the last layer means the level of interpretability still falls short of purely rules-based interpretable methods~\citep{rudin2019stop}. Replacing the last linear combination layer with an even more interpretable strategy is left to future work.

\section{Conclusion}
\label{sec:conclusion}
Our proposed SparXnet provides an innovative, theoretically guaranteed, and empirically effective approach in the field of sparse estimation. We have introduced a parsimonious neural network architecture and a training procedure for feature selection in applications with small datasets and high interpretability requirements. Our architecture comprises a softmax layer that selects individual features, trainable 1-dimensional Lipschitz functions, and a final linear layer. The Lipschitz functions are learned by neural networks.  We show that the sample complexity of SparXnet only scales like the number of chosen features, with a logarithmic dependence on the total number of features involved in the inputs. In addition, the sample complexity is independent of the number of parameters involved in each transformation function, as long as those functions are Lipschitz. We have demonstrated through synthetic data experiments that SparXnet can successfully select the right feature and recover the ground truth function in controlled settings. In real-life datasets, we show that SparXnet can equal or surpass alternatives, despite considerably reduced model complexity and significantly improved interpretability.





\bigskip
\noindent 
\bibliographystyle{abbrvnat}

\bibliography{bibliography}

\appendix

\title{Appendix}

	\onecolumn

\renewcommand{\thesection}{{\Alph{section}}}
\renewcommand{\thesubsection}{\Alph{section}.\arabic{subsection}}
\renewcommand{\thesubsubsection}{\Roman{section}.\arabic{subsection}.\arabic{subsubsection}}
\setcounter{secnumdepth}{-1}
\setcounter{secnumdepth}{3}

\section{Missing proofs}
\subsection{Proof of Main Theorem}
\label{sec:proofmain}

Recall the following theorem on the covering number of classes of Lipschitz functions. 
\begin{theorem}[Covering number of Lipschitz function balls, see~\cite{von2004distance}, Theorem 17 page 684, see also~\cite{Tikhomirov1993}]
	
	Let $\mathcal{X}$ be a connected and centred metric space (i.e. for all $A\subset \mathcal{X}$ with $\diam(A)\leq 2r$ there exists an $x\in \mathcal{X}$ such that $d(x,a)\leq r$ for all $a\in A$).  Let $B$ denote the set of $1$-Lipschitz functions from $\mathcal{X}$ to $\R$  which are uniformly bounded by $\diam (\mathcal{X})$.
	For any $\epsilon>0$ we have the following bound on the covering number of the class $B$ as a function of the covering number of $\mathcal{X}$:
	\begin{align}
		\mathcal{N}\left(B,\epsilon,\|\nbull\|_{\infty}\right)\leq \left(   \left\lceil \frac{2\diam(\mathcal{X})}{\epsilon}\right \rceil+1\right) \times 2^{\mathcal{N}(\mathcal{X},\frac{\epsilon}{2},d)}.
	\end{align}
	
\end{theorem}

Applying the above one dimensional space, we immediately obtain: 
\begin{proposition}
	\label{prop:liplinf}
	Let $\mathcal{F}_L$ denote the set of $L$-Lipschitz functions from $[-C,C]$ to $[-CL,CL]$.   
	We have the following bound on the covering number of $\mathcal{F}_L$ with respect to the $L^\infty$ (uniform) norm on functions: 
	\begin{align}
		\log\left(\mathcal{N}(\mathcal{F}, \epsilon, \|\nbull\|_{\infty})\right)&\leq \log \left(   \left\lceil \frac{4CL}{\epsilon}\right \rceil+1\right) +    \left[\left\lceil\frac{2CL}{\epsilon}\right\rceil+1\right]\log(2)\nonumber \\
		&\leq  \left[\frac{2CL}{\epsilon}+2\right]\log\left( \frac{4CL}{\epsilon} +2\right).
	\end{align}
\end{proposition}
\begin{proof}
	
	Follows from taking the following cover of $[-C,C]$: $\left[ [-C,C]\cap \epsilon  \mathbb{Z} \right]$, which has cardinality less than $\left[\left\lceil\frac{2C}{\epsilon}\right\rceil+1\right]$.

\end{proof}

Recall the following classic proposition:
\begin{proposition}[cf.~\citep{bookhighprob,bartlet98,pisier,Ledent_21_Norm-based}]
	\label{posthocify}
	Let $B_\beta$ denote the ball of radius $\beta$ in $\mathbb{R}^d$ with respect to the $L^1$ norm.  We have
	\begin{equation}
		\log\left(\mathcal{N}(B_{\beta},\epsilon,\|\nbull\|_{2})\right)\leq \left \lceil\frac{\beta^2}{\epsilon^2}\right \rceil\log(2d)
	\end{equation}

\end{proposition}

In the case of an  $L^\infty$ norm over the  samples, the  following  much deeper result holds~\citep{ZhangCover} (Theorem 4, page 537, cf. also~\cite{Ledent_21_Norm-based}, Proposition 4, p. 8284): 
\begin{proposition}
	\label{MaureySup}
	Let $N,d\in\mathbb{N}$, $a,b>0$. Suppose we are given $N$ data points collected as the rows of a matrix $X\in \mathbb{R}^{N\times d}$, with $\|X_{i,\nbull}\|_2\leq b,\forall i=1,\ldots,N$. For $U_{a,b}(X)=\big\{X\alpha:\|\alpha\|_2\leq a,\alpha\in\rbb^d\big\}$, we have
	\small 
	\begin{align*}
		\log_2\mathcal{N}\left(U_{a,b}(X),\epsilon,\|\nbull\|_{\infty}  \right)\leq \frac{36a^2b^2}{\epsilon^2}\log_2\left(\frac{8abN}{\epsilon}+6N+1\right).
	\end{align*}
\end{proposition}

\begin{proposition}

	Consider the following function class $\mathcal{F}:=$
	\begin{align}
 \label{eq:defineF}
		&\Bigg\{ F\in\R^d\rightarrow \R: F(x)=    \sum_{k=1}^K \theta_k f_k\left( \sum_{u=1}^d W^k_u x_u    \right) \> \text{for}  \>  \|f_k\|_{\lip} \leq L  (\forall k),\> \sum_{k=1}^K |\theta_k|\leq \Gamma,  \> W_u^k>0 \> (\forall u,k), \> \sum_{u=1}^dW_u^k \leq 1 (\forall k   )     \Bigg\}.\nonumber
	\end{align}
	Assume $K\leq d$.
	Consider a dataset $S=x^1,\ldots,x^N$ such that $|x^i|_{\max}\leq \chi$  for all $i\leq N$.  
	We have the following bound on the $L^\infty$ covering number of $\mathcal{F}$: 	
	\begin{align}
 \log\mathcal{N}_{\infty}\left(    \mathcal{F} ,\epsilon, S \right) \leq
	 \left[\frac{216K\chi^2  L^2 \Gamma^2 }{\epsilon^2}+3\right]\log_2\left(\frac{12\chi LdN}{\epsilon} +6dN+1\right).
	\end{align}
\end{proposition}

\begin{proof}

	Let $\bar{\mathcal{W}}:=\left\{ w\in\R^{d}:\sum_{u}w_u\leq 1  \right \} $ and let $\Theta$ denote the set of admissible $\theta$s. Let $\mathcal{W}=\bigotimes_{k=1}^K\bar{\mathcal{W}}$ denote the set of admissible $W$s. 
	
	By applying Proposition~\ref{MaureySup} we can find a cover $\bar{\mathcal{C}}_w\subset\bar{W}$ such that for all $w\in \bar{\mathcal{W}}$, there exists a $\bar{w}\in\bar{\mathcal{W}}$ such that for any $u=1,2,\ldots,d$, we have 
	\begin{align}
		\label{eq:firstguaranteepre}
		\left| 	\left\langle e_u,w-\bar{w}\right\rangle \right|  \leq \frac{\epsilon}{3L\chi  \Gamma}
	\end{align}
	and  
	\begin{align}
		\label{eq:firstcountpre}
		\log_2 (\#(\bar{\mathcal{C}}_w))\leq \frac{108\chi^2  L^2 \Gamma^2  }{\epsilon^2}\log_2\left(\frac{8\chi L  d \Gamma}{\epsilon}+6d+1\right).
	\end{align}
	From this, it also follows naturally that we can define a cover $\mathcal{C}_w:=\bigotimes_{k=1}^K \bar{C}_w$of the whole space $\mathcal{W}$ as $\bigotimes_{k=1}^K \bar{W}$, where for any $W\in\mathcal{W}$ we define $\bar{W}=\left(\bar{W}^1,\bar{W}^2,\ldots,\bar{W}^K\right)$
	
	Note that equation~\eqref{eq:firstguaranteepre} guarantees that for any $x\in S$ we have 
	\begin{align}
		\label{eq:firstguaranteereal}
		\left| 	\left\langle x,W-\bar{W}\right\rangle \right|_{\max}&=\min_k 	\left| 	\left\langle x,W^k-\bar{W}^k\right\rangle \right|\nonumber \\&\leq \frac{\epsilon}{3L\chi }\chi \sum_u(W^k_u)=\frac{\epsilon}{3L}.
	\end{align}
	
	Trivially, we  also have  from Equation~\eqref{eq:firstcountpre}:
	\begin{align}
		\label{eq:firstcountreal}
		\log_2 (\#(\mathcal{C}_w))\leq \frac{108K\chi^2  L^2 \Gamma^2  }{\epsilon^2}\log_2\left(\frac{8\chi L  d \Gamma }{\epsilon}+6d+1\right).
	\end{align}
	
	Next, by Proposition~\ref{prop:liplinf}, there exists an $\epsilon/3\Gamma$-uniform cover $\mathcal{C}_f$ of the space  $\{ f:  [-\chi,\chi]\rightarrow [-\chi L ,\chi L ]: \|f\|_{\lip}\leq L\}$ with cardinality satisfying 
	\begin{align}
		\label{eq:secondcount}
		\log_2\left(    \mathcal{C}_f   \right)
		&\leq  \left[\frac{6\chi L\Gamma}{\epsilon}+2\right]\log_2 \left( \frac{12\chi L\Gamma}{\epsilon} +2\right).
	\end{align}

	For any $x$ with $\|x\|_{\max}\leq \chi$, any $W\in\mathcal{W}$ and any $f_1,\ldots,f_K$ with $\|f_k\|_{\lip}\leq L$ ($\forall k$) we certainly have   $\left|f_k\left(\left\langle  W^k,x\right\rangle\right )\right |\leq \chi L $, which implies  $\left\|\left(f_1(\langle W^1,x^i\rangle),\ldots,f_K(\langle W^K,x^i\rangle) \right)\right\|_2\leq \sqrt{K}\chi$.  Thus, another application of Proposition~\ref{MaureySup} with $\mathcal{Y}_{W,f}$ as a "training set", there exists a cover $\mathcal{C}_{\theta,W,f}
	$ of the space $\Theta$ such that for all  $\theta\in\Theta$,  there exists a $\bar{\theta}\in\mathcal{C}_{\theta,W,f}$  such that  for all  $y$ in $\mathcal{Y}_{w,f}$, 
	\begin{align}
		\label{eq:thirdguarantee}
		\left|\left\langle 	\theta -\bar{\theta}, y\right\rangle \right|\leq \frac{\epsilon}{3}.
	\end{align}
	and 
	\begin{align}
		\label{eq:thirdcount}
		\log_2( \#\mathcal{C}_{\theta,w,f})&\leq   \frac{108K\Gamma^2\chi^2  L^2 }{\epsilon^2}\log_2\left(\frac{8\Gamma \chi L \sqrt{K} N}{\epsilon}+6N+1\right).
	\end{align} 
	
	We now take our final cover of $\mathcal{F}$ as follows: 
	\begin{align}
		&\mathcal{C}= \Bigg\{  F\in\R^d\rightarrow \R: F(x)=    \sum_{k=1}^K \theta_k f_k\left( \sum_{u=1}^d W^i_u x_u    \right) :   \nonumber  \\ &\quad\quad\quad\quad\quad\quad\quad \quad \quad w\in\mathcal{C}_w, f_1,\ldots,f_K \in\mathcal{C}_f,\theta\in\mathcal{C}_{\theta,w,f}   \Bigg\}\nonumber .
	\end{align}
	
	For any $F\in\mathcal{F}$ let $\bar{F}:=  \sum_{k=1}^K \bar{\theta}_k \bar{f}_k\left( \sum_{u=1}^d \bar{W}^i_u x_u    \right)$   where  $\bar{W}$ (resp. $\bar{f}$, $\bar{\theta}$) is the cover element of  $\mathcal{C}_w$ (resp. $\mathcal{C}_f$, $\mathcal{C}_{\theta,\bar{f},\bar{W}}$) associated to $W$ (resp. $f$,$\theta$). We clearly have 
	\begin{align}
			\left|  F(x)-\bar{F}(x)\right|&\leq   \left|F(x)\sum_{k=1}^K \theta_k f_k\left( \sum_{u=1}^d \bar{W}^i_u x_u    \right) \right|  +\left|  \sum_{k=1}^K \theta_k f_k\left( \sum_{u=1}^d \bar{W}^i_u x_u    \right) -\sum_{k=1}^K \theta_k \bar{f}_k\left( \sum_{u=1}^d \bar{W}^i_u x_u    \right) \right|\nonumber \\
		&\quad \quad \quad \quad \quad \quad \quad \quad\quad \quad \quad \quad\quad \quad \quad \quad\quad \quad \quad \quad + \left|\sum_{k=1}^K \theta_k \bar{f}_k\left( \sum_{u=1}^d \bar{W}^i_u x_u    \right)  -\bar{F}(x)\right|\\
		&\leq \frac{\epsilon}{3L\Gamma}\|f\|_{\lip} \left(\sum \theta_k\right)+\left (\sum \theta_k\right ) \frac{\epsilon}{3\Gamma }+\frac{\epsilon}{3} \\& \leq \epsilon.
	\end{align}
	Thus, $\mathcal{C}$ is indeed an $\epsilon$-cover of $\mathcal{F}$.

	Furthermore, by Equations~\eqref{eq:firstcountreal}, ~\eqref{eq:secondcount}, ~\eqref{eq:thirdcount}  we have 
	\begin{align}
		\log_2\left(\#(\mathcal{C})\right) &\leq \frac{216K\chi^2  L^2 \Gamma^2 }{\epsilon^2}\log_2\left(\frac{8\chi L  dN}{\epsilon}+6dN+1\right)  + \left[\frac{6\chi L \Gamma }{\epsilon}+2\right]\log_2 \left( \frac{12\chi L\Gamma }{\epsilon} +2\right) \\
		&\leq \left[\frac{216K\chi^2  L^2  \Gamma ^2}{\epsilon^2}+3\right]\log_2\left(\frac{12\chi  \Gamma LdN}{\epsilon} +6dN+1\right),
	\end{align}  
	as expected. 
	
\end{proof}

Next, we apply Dudley's entropy formula (Proposition~\ref{prop:dudley}) to bound the Rademacher complexity of the function class $\mathcal{F}$ as defined above. 

\begin{corollary}
	\label{cor:rademabound}
	For any training set $S$ of size $N$ we  have the following bound on the Rademacher complexity of the function class $\mathcal{F}$  defined above: 
	\begin{align}
		\rad_S(\mathcal{F})\leq  \frac{12}{\sqrt{N}} \left[15\chi L\Gamma \sqrt{K}+3\right]\sqrt{\log_2\left(12dN^2\left[\chi L\Gamma  +1\right]\right) }\log(N).
	\end{align}
\end{corollary}
\begin{proof}
	Applying Proposition~\ref{prop:dudley}  with $\alpha=\frac{1}{N}$  and $p=\infty$ we obtain:  
	\begin{align}
		   \rad(\mathcal{F})&\leq \frac{4}{N}+\frac{12}{\sqrt{N}}  \int_{\frac{1}{N}}^1      \sqrt{\log \mathcal{N}(\mathcal{F}|S,\epsilon,\|\nbull\|_{\infty}) }      \\
		&\leq \frac{12}{\sqrt{N}}  \int_{\frac{1}{N}}^1  \sqrt{\left[\frac{216K\chi^2  L^2  \Gamma ^2}{\epsilon^2}+3\right]\log_2\left(\frac{12\chi  \Gamma LdN}{\epsilon} +6dN+1\right)} d\epsilon  
		+\frac{4}{N}  \\
		&\leq \frac{12}{\sqrt{N}}\left[1+2\sqrt{\log_2\left(12dN^2\chi L\Gamma +6dN+1\right)} \right] \label{eq:prolemsfirst} + \frac{12}{\sqrt{N}}\int_{\frac{1}{N}}^1  \frac{15\chi L \Gamma \sqrt{K}}{\epsilon}\sqrt{\log_2\left(12dN^2\chi L\Gamma +6dN+1\right)}d\epsilon   \nonumber\\
		&\leq \frac{12}{\sqrt{N}} \left[15\chi L\Gamma \sqrt{K}+3\right]\sqrt{\log_2\left(12dN^2\left[\chi L\Gamma  +1\right]\right) }\log(N),
	\end{align}
	as expected. At equation~\eqref{eq:prolemsfirst}, we have used the fact $\frac{12}{\sqrt{N}}\geq \frac{4}{N}$ and replaced $\epsilon$ by its lower bound $1/N$ in the logarithms. 
\end{proof}

Finally, the above results allow us to prove our main theorem~\ref{thm:maintheorem}.

\begin{proof}[Proof of Theorem~\ref{thm:maintheorem}]
	Follows from a direct application of Corollary~\ref{cor:rademabound},  Talagrand's concentration Lemma~\ref{lem:talagrand} and Theorem~\ref{rademachh} using the classic lemma of Statistical Learning Theory. 
	
	Indeed, by the above results,  with probability $\geq 1-\delta$ over the draw of the training set, for any $f\in\mathcal{F}$,  we have:  
	\begin{align}
\left|	\ell(f)-\hat{\ell}(f)\right|\leq & \frac{12\mathcal{L}}{\sqrt{N}} \left[15\chi L\Gamma \sqrt{K}+3\right]\sqrt{\log_2\left(12dN^2\left[\chi L\Gamma  +1\right]\right) }\log(N) +3B\sqrt{\frac{\log(2/\delta)}{2N}}
	\end{align}
	for all $f\in\mathcal{F}$, where $\ell(f):=\E_{x,y} \ell\left(f(x),y\right)$ and $\hat{\ell}(f)= \hat{\E}_{x,y} \ell\left(f(x),y\right) =\frac{1}{N}\sum_{i=1}
	^N	  \ell\left(f(x^i),y^i\right) $. 
	
	Next, under the same high probability event, we have:
	\begin{align}
		&\ell(\hat{f})-\ell(f^*) = 	\ell(\hat{f})-\hat{\ell}(\hat{f})+ \hat{\ell}(\hat{f}) -\hat{\ell}(f^*)+ \hat{\ell}(f^*) -\ell(f^*)   \\& \leq \ell(\hat{f})-\hat{\ell}(\hat{f})+ \hat{\ell}(f^*) -\ell(f^*)  \nonumber \\ &\leq \frac{24\mathcal{L}}{\sqrt{N}} \left[15\chi L\Gamma \sqrt{K}+3\right]\sqrt{\log_2\left(12dN^2\left[\chi L\Gamma  +1\right]\right) }\log(N)+6B\sqrt{\frac{\log(2/\delta)}{2N}},
	\end{align}
	as expected.
\end{proof}

\begin{proof}[Proof of Corollary~\ref{cor:maincor}]
By a direct application of Theorem~\ref{thm:maintheorem}, we have that the excess risk is bounded with high probability as \\ $O\left(\sqrt{\frac{KL^2}{N}\log(d+L+1)\log^2(N)}\right)$. This implies that an excess risk of $\epsilon$ or less can be reached (w.h.p.) as long as 
\begin{align}
\frac{N}{\log^2(N)}\geq O\left(\frac{KL^2 \log(d+L+1)}{\epsilon^2}\right).
\end{align}
This in turn is satisfied as long as 
\begin{align}
N\geq O\left( \frac{KL^2 \log(d+L+1)}{\epsilon^2} \log^2\left(\frac{KL^2 \log(d+L+1)}{\epsilon^2} \right)    \right) \\
\geq O\left( \frac{KL^2}{\epsilon^2} \log^3\left(\frac{KL^2 \log(d+L+1)}{\epsilon^2} \right)  \right),
\end{align}
as expected. 
Indeed, if $x=y\log^2(y)$ and $x,y\geq 1$ we have by direct calculation
	\begin{align}
	\frac{x}{\log^2(x)}&=\frac{y\log^2(y)}{[\log(y)+\log(\log^2(y))]^2}\nonumber \\
	&\geq \frac{y\log(y)}{\log(y)+2\log(y)}\nonumber \\
	&=\frac{y}{3},
	\end{align}
	where at the second line we have used the inequality $\log(y)\leq y$.
\end{proof}

\subsection{Some classic lemmas}
In this subsection, we recall some classic known results which are required to prove our bounds.  This is purely for the reader's convenience and no claim of originality is made. 
Recall the definition of the Rademacher complexity of a function class $\mathcal{F}$:
\begin{definition}
	Let $\mathcal{F}$ be a class of real-valued functions with range $X$. Let also $S=(x_1,x_2,\ldots,x_n)\in X$ be $n$ samples from the domain of the functions in $\mathcal{F}$. The empirical Rademacher complexity $\rad_S(\mathcal{F})$ of $\mathcal{F}$ with respect to $x_1,x_2,\ldots,x_n$ is defined by
	\begin{align}
		\rad_S(\mathcal{F}):=\mathbb{E}_{\delta}\sup_{f\in\mathcal{F}}\frac{1}{n}\sum_{i=1}^n  \delta_if(x_i),
	\end{align}
	where $\delta=(\delta_1,\delta_2,\ldots,\delta_n)\in\{\pm 1\}^n$ is a set of $n$ iid Rademacher random variables (which take values $1$ or $-1$ with probability $0.5$ each).
\end{definition}

Recall the following classic theorem~\citep{rademach}:
\begin{theorem}
	\label{rademachh}
	Let $Z,Z_1,\ldots,Z_n$ be iid random variables taking values in a set $\mathcal{Z}$. Consider a set of functions $\mathcal{F}\in[0,1]^{\mathcal{Z}}$. $\forall \delta>0$, we have with probability $\geq 1-\delta$ over the draw of the sample $S$ that $$\forall f \in \mathcal{F}, \quad \mathbb{E}(f(Z))\leq \frac{1}{n}\sum_{i=1}^nf(z_i)+2\rad_{S}(\mathcal{F})+3\sqrt{\frac{\log(2/\delta)}{2n}}. $$
\end{theorem}

We will also need the following result (Dudley's entropy formula~\citep{Spectre,Ledent_21_Norm-based})

\begin{proposition}
	\label{prop:dudley}
	Let $\mathcal{F}$ be a real-valued function class taking values in $[0,1]$, and assume that $0\in \mathcal{F}$. Let $S$ be a finite sample of size $n$. For any $2\leq p\leq \infty$, we have the following relationship between the Rademacher complexity  $\rad(\mathcal{F}|_{S})$ and the covering number $\mathcal{N}(\mathcal{F}|S,\epsilon,\|\nbull\|_{p})$.
	\begin{align*}
		\rad(\mathcal{F}|_{S})\leq \inf_{\alpha>0} \left(     4\alpha+\frac{12}{\sqrt{n}}  \int_{\alpha}^1      \sqrt{\log \mathcal{N}(\mathcal{F}|S,\epsilon,\|\nbull\|_{p}) }     \right),
	\end{align*}
	where the norm $\|\nbull\|_{p}$ on $\mathbb{R}^m$ is defined by $\|x\|_{p}^p=\frac{1}{n}(\sum_{i=1}^m|x_i|^p)$.
\end{proposition}

\begin{lemma}[Talagrand contraction lemma (cf.~\cite{ledoux91probability} see also~\cite{meir03} page 846)]
	\label{lem:talagrand}
	
	Let $g:\R\rightarrow \R$ be 1-Lipschitz. 
	Consider the set of functions  $\left\{ f_i(\theta) , i\leq N \right\}$  (on $\{1,2,\ldots,N\}$) depending on a parameter $\theta\in\Theta$.

We have
	\begin{align}
		\E_\sigma \sup_{\theta\in\Theta} \left\{ \sum_{i=1}^N\sigma_i g(f_i(\theta))      \right\} \leq \E_\sigma  \E_X \sup_{\theta\in\Theta} \left\{  \sum_{i=1}^N \sigma_i f_i(\theta) \right\},
	\end{align}
	where the $\sigma_i$s are i.i.d. Rademacher variables.	
\end{lemma}

\begin{table}[h!] \small
	\centering
	\begin{tabular}{c|c}
		Notation & Meaning  \\
		\midrule
  $\mathcal{N}(\nbull)$ & Covering number of function class \\ \midrule
  $N$ & Number of samples\\
  $d$ & total number of features\\
  $K$ & number of features to be selected\\
  $x^i$ & $i$th sample\\
  $w^k_u$ & weight for $k$th node and $u$th input \\ & (before softmax)\\
  $W$& 1st layer weights after softmax\\
  $f_k:\R\rightarrow \R $ & transformation function for $k$th node\\
  $\theta$ & weights at the last linear layer \\
  $F$ & prediction function\\
    $\tau$  & Temperature parameter\\
  \midrule
  $\chi$ & Upper bound on $\|x\|_{\max}$\\
  $L$ & Upper bound on the Lipschitz constant of the  $f_k$s\\
    $\ell$ & Loss function \\  $B$ & Upper bound on loss function\\
  $\mathcal{L}$ & Upper bound on Lipschitz Constant of loss function\\
  $\Gamma$ & Upper bound on $\sum_k \theta_k$\\
  $\mathcal{F}$ & Function class defined in~\eqref{eq:defineF}\\
  \end{tabular}
  \end{table}

\section{Additional experiment details}

\subsection{Description of real datasets}

We entertain six different datasets in our analysis. These datasets provide diverse challenges in data preprocessing, feature selection, and model evaluation, enabling us to showcase the flexibility and robustness of our approaches.

\begin{itemize}
    \item The \textbf{Heart Disease dataset} from the UCI machine learning repository includes 303 instances with 14 features, used to classify individuals into heart disease categories. The features include age, sex, chest pain type, resting blood pressure, serum cholesterol, fasting blood sugar, and others related to cardiac conditions.

    \item The \textbf{Adult Income dataset} predicts whether income exceeds \$50K/yr based on census data, with a total of 14 features such as age, work class, education, marital status, occupation, race, gender, and native country. The binary target variable distinguishes between high and low income.

    \item The \textbf{Breast Cancer Wisconsin dataset} is a classification task involving 569 instances with 30 numeric features. This dataset is used to predict whether a tumor is malignant or benign based on measurements such as the mean radius, mean texture, mean smoothness, and other cellular properties.

    \item The \textbf{Recidivism dataset} from the COMPAS tool provides data on recidivism risk, which is used to classify individuals based on the likelihood of reoffending within two years. The dataset includes various features related to a defendant's criminal history, demographics, and COMPAS score.

    \item The \textbf{Customer Churn dataset} predicts customer churn based on 21 characteristics derived from the data of a telecommunications company, including tenure, monthly charges, total charges, and various categorical features related to customer services and demographics. The target variable indicates whether a customer has churned.

    \item The \textbf{Credit Bureau dataset}, as mentioned earlier, involves the classification task of predicting loan default based on 7 features related to the personal and financial characteristics of borrowers.
\end{itemize}

For each dataset, our code performs standard preprocessing steps, such as handling missing values, encoding categorical variables, and scaling numeric features. We further perform stratified sampling when splitting the data into training and testing sets.

Finally, we utilize pipelines that integrate these preprocessing steps with model fitting, ensuring that our approach is consistent and reproducible across different datasets. This setup allows us to evaluate various machine learning models effectively, making it a robust framework for both academic research and practical applications.

\subsection{Impact of number of pathway on out-of-sample AUC}

As depicted in Figure \ref{fig:num_pathway_realdata}, there is generally an increasing trend in out-of-sample AUC as more features are incorporated. However, this trend is not strictly monotonic. Variations across different datasets suggest that the sensitivity to the number of pathways can differ significantly. Given the inherent uncertainty in the optimal number of features to select in real-world datasets, this analysis provides a data-driven guide to determining the most effective number of features to include in the model.

\begin{figure}[htbp]
    \centering
    \includegraphics[width=0.75\linewidth]{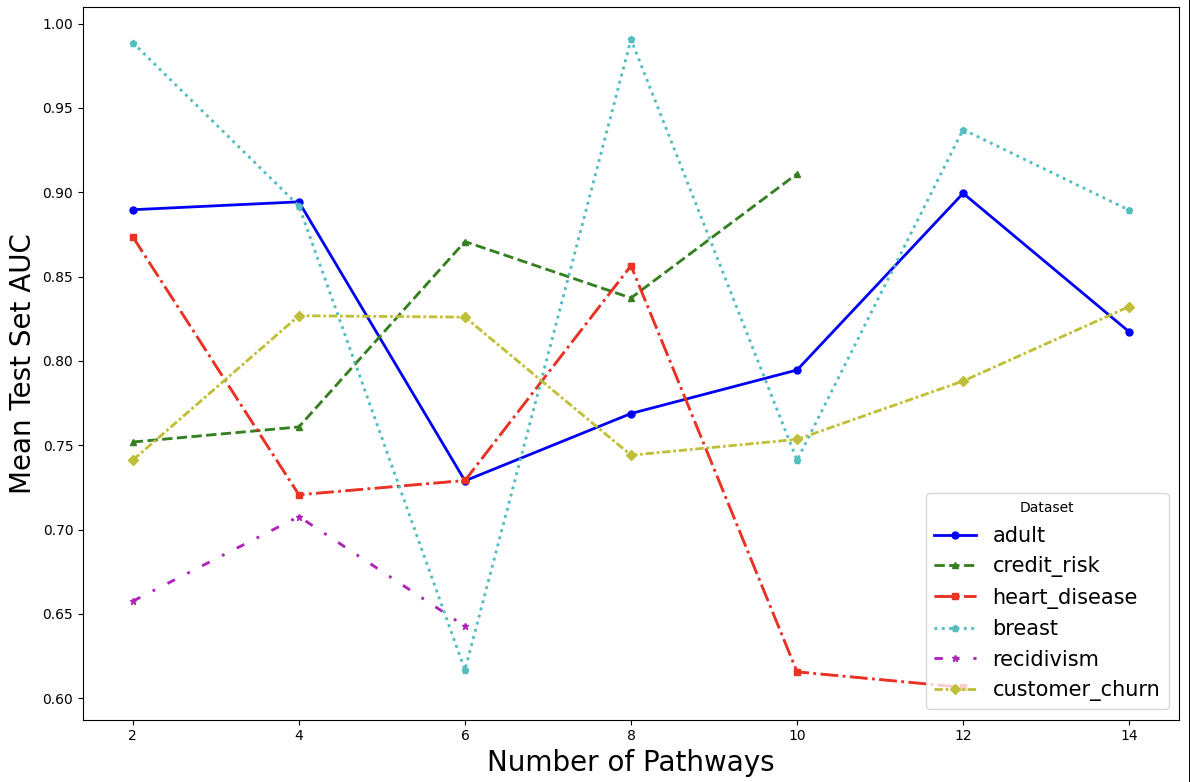}
    \caption{Comparison of mean test AUC by the number of pathways across datasets. The analysis highlights the variability in performance sensitivity to the number of pathways, serving as a guide for optimal feature selection in practical applications.}
    \label{fig:num_pathway_realdata}
\end{figure}

\subsection{Model Saturation}

The interpretability of our model hinges on the selection of a subset of predictive features through the softmax transformation of the weights between the input layer and the first hidden layer. As illustrated in Figure \ref{fig:weight_saturation}, nearly all the learned weights between these two layers are saturated, effectively considering only one input feature for each pathway. In this context, five features are selected: loan grade, loan percent income, loan intent, loan interest rate, and personal income, with the loan interest rate feature being selected by two pathways.

\begin{figure}
    \centering
    \includegraphics[width=0.75\linewidth]{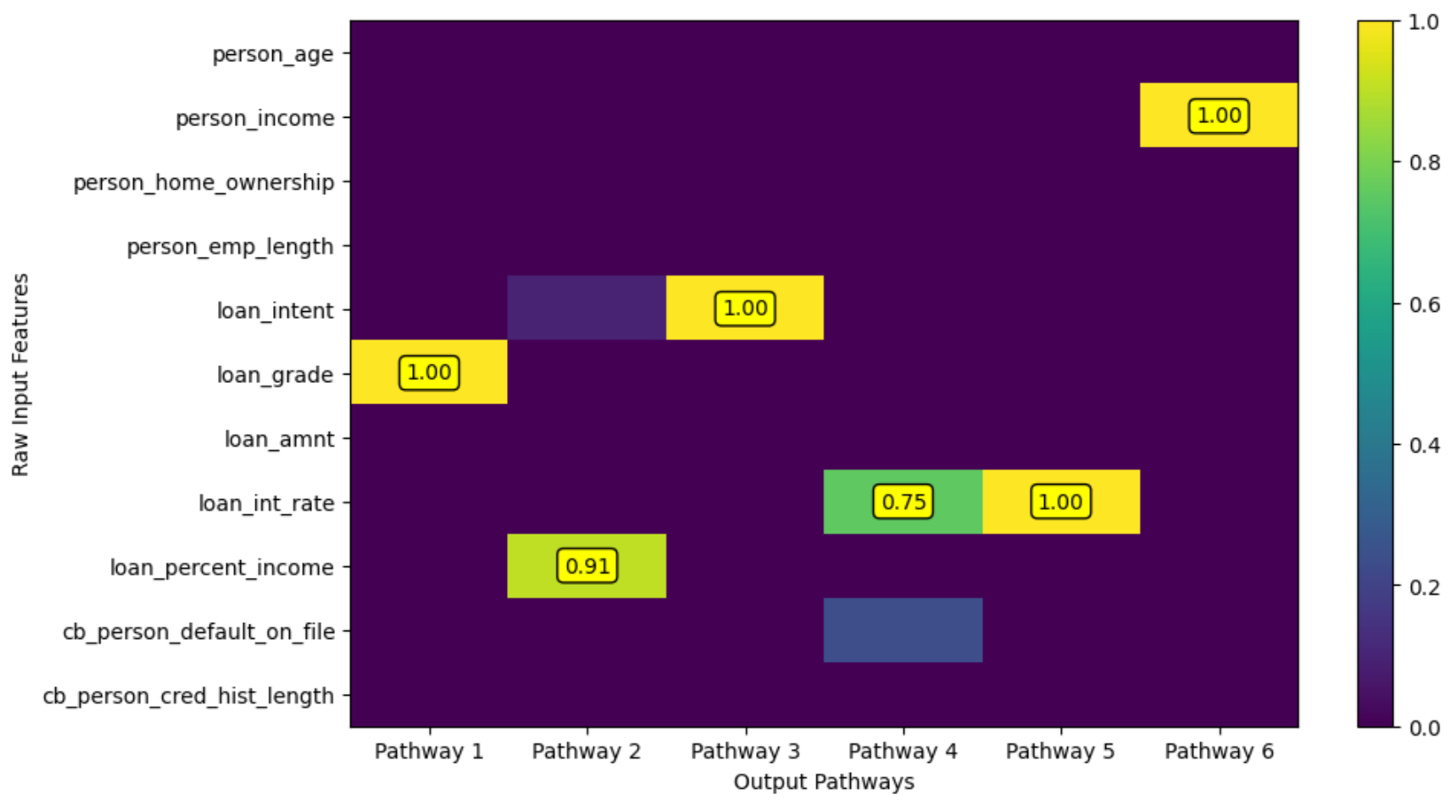}
    \caption{Heatmap of weights between the input layer and the first hidden layer after the softmax operation. The figure indicates the selected features for each pathway, completing the feature selection process. Most weights are saturated, being exactly one, with only pathways two and four approaching one.}
    \label{fig:weight_saturation}
\end{figure}

\end{document}